\documentclass[10pt,twocolumn,letterpaper]{article}

\usepackage[utf8x]{inputenc}
\usepackage{times}
\usepackage{epsfig}
\usepackage{graphicx}
\usepackage{float}
\usepackage{wrapfig}
\usepackage{amsmath,amssymb,amsthm}
\usepackage{algorithm,algorithmicx,algpseudocode}
\usepackage{bm,xspace}
\usepackage{comment}
\usepackage{verbatim}
\usepackage{multirow}
\usepackage{balance}
\usepackage{url}
\usepackage{booktabs}
\usepackage{etoolbox,siunitx}
\usepackage{calc}
\usepackage{pifont,hologo}
\usepackage[usenames, dvipsnames]{color}
\usepackage{nicefrac}

\usepackage{epigraph}
\usepackage{enumitem}
\usepackage{soul}
\usepackage{subfigure}
\usepackage{mdframed}
\usepackage{microtype}

\newcommand{\ra}[1]{\renewcommand{\arraystretch}{#1}}

\usepackage[super]{nth}
\usepackage{nicefrac}
\robustify\bfseries

\def\aa{\mathbf{a}}
\def\bb{\mathbf{b}}
\def\cc{\mathbf{c}}

\def\sss{\mathbf{s}}

\def\bA{\mathbf{A}}

\def\bB{\mathcal{B}}
\def\cC{\mathcal{C}}

\def\fF{\mathcal{F}}

\def\iI{\mathcal{I}}

\def\oO{\mathcal{O}}

\def\sS{\mathcal{S}}

\def\vV{\mathcal{V}}

\def\Ii{\mathbb{I}}

\def\Re{\mathbb{R}}

\def\btheta{{\bm\theta}}

\newcommand\llesser{\mathbin{<\!\!\!<}}

\DeclareMathSymbol{@}{\mathord}{letters}{"3B}

\newcommand\norm[1]{\left\lVert#1\right\rVert}

\newtheorem{definition}{Definition}
\newtheorem{property}{Property}

\newtheorem{lemma}{Lemma}

\def\latex/{\LaTeX}
\def\bibtex/{\hologo{BibTeX}}

\usepackage{iccv}

\usepackage[breaklinks=true,bookmarks=false]{hyperref}

\iccvfinalcopy 

\def\iccvPaperID{****} 

\ificcvfinal\fi

\begin{document}

\title{Consensus Maximization Tree Search Revisited}

\author{Zhipeng Cai\\
The University of Adelaide
\and
Tat-Jun Chin\\
The University of Adelaide
\and
Vladlen Koltun\\
Intel Labs
}

\maketitle
\ificcvfinal\thispagestyle{empty}\fi

\begin{abstract}
	Consensus maximization is widely used for robust fitting in computer vision. However, solving it exactly, i.e., finding the globally optimal solution, is intractable. A* tree search, which has been shown to be fixed-parameter tractable, is one of the most efficient exact methods, though it is still limited to small inputs. We make two key contributions towards improving A* tree search. First, we show that the consensus maximization tree structure used previously actually contains paths that connect nodes at both adjacent and non-adjacent levels. Crucially, paths connecting non-adjacent levels are redundant for tree search, but they were not avoided previously. We propose a new acceleration strategy that avoids such redundant paths. In the second contribution, we show that the existing branch pruning technique also deteriorates quickly with the problem dimension. We then propose a new branch pruning technique that is less dimension-sensitive to address this issue. Experiments show that both new techniques can significantly accelerate A* tree search, making it reasonably efficient on inputs that were previously out of reach. Demo code is available at \url{https://github.com/ZhipengCai/MaxConTreeSearch}.
\end{abstract}

\vspace{-1.1em}
\section{Introduction}\label{sec:intro}
The prevalence of outliers makes robust model fitting crucial in many computer vision applications. One of the most popular robust fitting criteria is \emph{consensus maximization}, whereby, given outlier-contaminated data $\sS = \{\sss_i\}_{i=1}^N$, we seek the model $\btheta \in \Re^d$ that is consistent with the largest subset of the data. Formally, we solve
\begin{align}~\label{obj:maxcon}
\begin{aligned}
& \underset{\btheta }{\text{maximize}}
& & c(\btheta | \sS) = \sum_{i=1}^N{\Ii\{r(\btheta | \sss_i)\le\epsilon\}},
\end{aligned}
\end{align}
where $c(\btheta | \sS)$ is called the \emph{consensus} of $\btheta$. The 0/1 valued indicator function $\Ii\{\cdot\}$ returns 1 only when $\sss_i$ is consistent with $\btheta$, which happens when the residual $r(\btheta | \sss_i)\le\epsilon$. The form of $r(\btheta | \sss_i)$ will be defined later in Sec.~\ref{sec:treeSearch}. Constant $\epsilon$ is the predefined inlier threshold, and $d$ is called the ``problem dimension". Given the optimal solution $\btheta^*$ of~\eqref{obj:maxcon}, $\sss_i$ is an inlier if $r(\btheta^*|\sss_i)\le \epsilon$ and an outlier otherwise.

Consensus maximization is NP-hard~\cite{Chin18}, hence, sub-optimal but efficient methods are generally more practical. Arguably the most prevalent methods of this type are RANSAC~\cite{fischler81} and its variants~\cite{chum2003locally,tordoff2005guided,chum2005matching,raguram13}, which iteratively fit models on randomly sampled (minimal) data subsets and return the model with the highest consensus. However, their inherent randomness makes these methods often distant from optimal and sometimes unstable. To address this problem, deterministic optimization techniques~\cite{purkait17,le17,cai18} have been proposed, which, with good initializations, usually outperform RANSAC variants. Nonetheless, a good initial solution is not always easy to find. Hence, these methods may still return unsatisfactory results.

The weaknesses of sub-optimal methods motivate researchers to investigate globally optimal methods; however, so far they are effective on only small input sizes (small $d$, $N$ and/or number of outliers $o$). One of the most efficient exact methods is tree search~\cite{li2007practical,Chin15,Chin17} (others surveyed later in Sec.~\ref{sec:rel}), which fits~\eqref{obj:maxcon} into the framework of the LP-type methods~\cite{sharir1992combinatorial,matousek95}. By using heuristics to guide the tree search and conduct branch pruning, A* tree search~\cite{Chin15,Chin17} has been demonstrated to be much faster than Breadth-First Search (BFS) and other types of globally optimal algorithms. In fact, tree search is provably fixed-parameter tractable (FPT)~\cite{Chin18}. Nevertheless, as demonstrated in the experiment of~\cite{Chin17} and later ours, A* tree search can be highly inefficient for challenging data with moderate $d$ ($\ge 6$) and $o$ ($\ge 10$).

\vspace{-1em}
\paragraph{Our contributions.}
In this work, we analyze reasons behind the inefficiency of A* tree search and develop improvements to the algorithm. Specifically:
\begin{itemize}[leftmargin=1em,itemsep=0em,topsep=0em,parsep=1pt]
	\item We demonstrate that the previous tree search algorithm does not avoid all redundant paths, namely, paths that connect nodes from non-adjacent levels. Based on this observation, a new acceleration strategy is proposed, which can avoid such non-adjacent (and redundant) paths.
	\item We show that the branch pruning technique in~\cite{Chin17} is not always effective and may sometimes slow down the tree search due to its sensitivity to $d$. To address this problem, we propose a branch pruning technique that is less dimension-sensitive and hence much more effective.
\end{itemize}
Experiments demonstrate the significant acceleration achievable using our new techniques (3 orders of magnitude faster on challenging data). Our work represents significant progress towards making globally optimal consensus maximization practical on real data.


\subsection{Related Work}\label{sec:rel}

Besides tree search, other types of globally optimal methods include branch-and-bound (BnB)~\cite{li09,zheng11,parrabustos15}, whose exhaustive search is done by testing all possible $\btheta$. However, the time complexity of BnB is exponential in the size of the parameter space, which is often large. Moreover, the bounding function of BnB is problem-dependent and not always trivial to construct. Another type of methods~\cite{olsson2008polynomial,enqvist2012robust} enumerate and fit models on all possible bases, where each basis is a data subset of size $p$, where $p \llesser N$ and $p$ is usually slightly larger than $d$, e.g., $p = d+1$. The number of all possible bases is ${N \choose p}$, which scales poorly with $N$ and $d$. Besides differences in actual runtime, what distinguishes tree search from the other two types of methods is that tree search is FPT~\cite{Chin18}: its worst case runtime is exponential in $d$ and $o$, but polynomial in $N$.

\section{Consensus maximization tree search}\label{sec:treeSearch}

We first review several concepts that are relevant to consensus maximization tree search.

\subsection{Application range}

Tree search requires the residual $r(\btheta|\sss_i)$ to be \emph{pseudo-convex}~\cite{Chin17}. A simple example is the linear regression residual
\begin{align}\label{eq:linefitting}
r(\btheta|\sss_i) = |\aa_i^T\btheta - b_i|,
\end{align}
where each datum $\sss_i = \{\aa_i, b_i\}$, $\aa_i\in\Re^d$ and $b_i\in \Re$. Another example is the residual used in common multiview geometry problems~\cite{olsson2007efficient,cai18}, which are of the form
\begin{align}\label{eq:pseudoConvex}
r(\btheta|\sss_i) = \frac{\norm{\bA_i^T\btheta - \bb_i}_p}{\cc_i^T\btheta-d_i},
\end{align}
where each datum $\sss_i = \{\bA_i, \bb_i, \cc_i, d_i\}$, $\bA_i\in\Re^{d\times m}, \bb_i\in \Re^m$, $\cc_i \in \Re^d$ and $d_i\in\Re$. Usually, $p$ is 1, 2 or $\infty$.

\subsection{LP-type problem}

The tree search algorithm for~\eqref{obj:maxcon} is constructed by solving a series of minimax problems, which are of the form
\begin{align}\label{obj:minimax}
& \underset{\btheta}{\text{minimize}}\  \max_{i \in \sS^1 }~r(\btheta | \sss_i).
\end{align}
Problem~\eqref{obj:minimax} minimizes the maximum residual for all data in $\sS^1$, which is an arbitrary subset of $\sS$. For convenience, we define $f(\sS^1)$ as the minimum objective value of~\eqref{obj:minimax} computed on data $\sS^1$, and $\btheta(\sS^1)$ as the (exact) minimizer.

Throughout the paper, we will assume that $r(\cdot)$ is pseudo-convex and $\sS$ is non-degenerate (otherwise infinitestimal perturbations can be applied to remove degeneracy~\cite{matousek95,Chin17}). Under this assumption, problem~\eqref{obj:minimax} has a unique optimal solution and can be solved efficiently with standard solvers~\cite{eppstein2005quasiconvex}. Furthermore, \eqref{obj:minimax} is provably an LP-type problem~\cite{sharir1992combinatorial,amenta1999optimal,eppstein2005quasiconvex}, which is a generalization of the linear programming (LP) problem. An LP-type problem has the following properties:
\begin{property}[\textbf{Monotonicity}]\label{prop:Mon}
	For every two sets ${\sS^1 \subseteq \sS^2 \subseteq \sS}$, ${f(\sS^1) \le f(\sS^2) \le f(\sS)}$.
\end{property}

\begin{property}[\textbf{Locality}]\label{prop:Loc}
	For every two sets $\sS^1 \subseteq \sS^2 \subseteq \sS$ and every $\sss_i \in \sS$, $f(\sS^1) = f(\sS^2) = f(\sS^2 \cup \{\sss_i\})$ $\Rightarrow f(\sS^1) = f(\sS^1 \cup \{\sss_i\}) $.
\end{property}

With the above properties, the concept of \emph{basis}, which is essential for tree search, can be defined.

\begin{definition}[\textbf{Basis}]\label{def:basis}
	A basis $\bB$ in $\sS$ is a subset of $\sS$ such that for every $\bB'\subset\bB$, $f(\bB') < f(\bB)$.
\end{definition}

For an LP-type problem~\eqref{obj:minimax} with pseudo-convex residuals, the  maximum size of a basis, which we call \emph{combinatorial dimension}, is $d+1$.

\begin{definition}[\textbf{Violation set, level and coverage}]\label{def:vio}
	The violation set of a basis $\bB$ is defined as $\vV(\bB)= \{\sss_i \in \sS | $ $r(\btheta(\bB) | \sss_i) > f(\bB)\}$. We call $l(\bB) = |\vV(\bB)|$ the level of $\bB$ and $\cC(\bB) = \sS \backslash \vV(\bB)$ the coverage of $\bB$.
\end{definition}

By the above definition,
\begin{align}
c(\btheta(\bB) | \sS) = |\sS| - l(\bB).
\end{align}
An important property of LP-type problems is that solving~\eqref{obj:minimax} on $\cC(\bB)$ and $\bB$ return the same solution.

\begin{definition}[\textbf{Support set}]\label{def:suppSet}
	The level-0 basis for $\sS$ is called the support set of $\sS$, which we represent as $\tau(\sS)$.
\end{definition}

Assume we know the maximal inlier set $\iI$ for~\eqref{obj:maxcon}, where $|\iI| = c(\btheta^\ast | \sS)$. Define $\bB^* = \tau(\iI)$ as the support set of $\iI$; $\bB^*$ can be obtained by solving~\eqref{obj:minimax} on $\iI$. Then, $l(\bB^*)$ is the size of the minimal outlier set. Our target problem~\eqref{obj:maxcon} can then be recast as finding the optimal basis
\begin{align}\label{obj:treeSearch}
& \bB^* = \underset{\bB \subseteq \sS}{\text{argmin}}\ \ l(\bB), \ \text{s.t.}\ f(\bB) \le \epsilon,
\end{align}
and $\btheta(\bB^*)$ is the maximizer of~\eqref{obj:maxcon}. Intuitively, $\bB^*$ is the lowest level basis that is feasible, where a basis $\bB$ is called feasible if $f(\bB) \le \epsilon$.

\subsection{A* tree search algorithm}\label{sec:A*}

Matou\v{s}ek~\cite{matousek95} showed that the set of bases for an LP-type problem can be arranged in a tree, where the root node is $\tau(\sS)$, and the level occupied by a node $\bB$ on the tree is $l(\bB) = |\vV(\bB)|$. Another key insight is that there exists a path from $\tau(\sS)$ to any higher level basis, where a path is formed by a sequence of \emph{adjacent bases}, defined as follows.

\begin{algorithm}[!htb]\centering
	\begin{algorithmic}[1]
		\Require $\sS = \{ \sss_i \}^{N}_{i=1}$, threshold $\epsilon$.
		\State Insert $\bB = \tau(\sS)$ with priority $e(\bB)$ into queue $q$.
		\State Initialize hash table $T$ to NULL.
		\While{$q$ is not empty}
		\State Retrieve from $q$ the $\bB$ with the lowest $e(\bB)$.
		\If{$f(\bB) \le \epsilon$}
		\State return $\bB^* = \bB$.
		\EndIf
		\State $\bB_r \leftarrow$ Attempt to reduce $\bB$ by TOD method. \label{line:A*-BP}
		\For{each $\sss \in \bB_r$}
		\If{indices of $\vV(\bB) \cup \{\sss\}$ do not exist in $T$}\label{line:A*-Rep}
		\State Hash indices of $\vV(\bB) \cup \{\sss\}$ into T.
		\State $\bB' \leftarrow \tau(\cC(\bB)\backslash\{\sss\})$.\label{line:A*-expand}
		\State Insert $\bB'$ with priority $e(\bB')$ into $q$.
		\EndIf
		\EndFor
		\EndWhile
		\State Return error (no inlier set of size greater than $p$).
	\end{algorithmic}
	\caption{A* tree search of Chin et al.~\cite{Chin17} for~\eqref{obj:treeSearch}}\label{alg:A*}
\end{algorithm}

\begin{definition}[\textbf{Basis adjacency}]\label{def:basisAdj}
	Two bases $\bB'$ and $\bB$ are adjacent if $\vV(\bB') = \vV(\bB) \cup \{\sss_i\}$ for some $\sss_i \in \bB$.
\end{definition}

\vspace{-0.15em}

Intuitively, $\bB'$ is a direct child of $\bB$ in the tree. We say that we ``follow the edge" from $\bB$ to $\bB'$ when we compute $\tau(\cC(\bB) \backslash \{\sss_i\})$. Chin et al.~\cite{Chin17} solve~\eqref{obj:treeSearch} by searching the tree structure using the A* shortest path finding technique (Algorithm~\ref{alg:A*}). Given input data $\sS$, A* tree search starts from the root node $\tau(\sS)$ and iteratively expands the tree until $\bB^*$ is found. The queue $q$ stores all unexpanded tree nodes. And in each iteration, a basis $\bB$ with the lowest \emph{evaluation} value $e(\bB)$ is expanded. The expansion follows the basis adjacency, which computes~$\tau(\cC(\bB)\backslash\{\sss\})$ for all $\sss \in \bB$ (Line~\ref{line:A*-expand} in Algorithm~\ref{alg:A*}).

The evaluation value is defined as
\begin{align}
e(\bB) = l(\bB) + h(\bB),
\end{align}
where $h(\bB)$ is a \emph{heuristic} which estimates the number of outliers in $\cC(\bB)$. A* search uses only \emph{admissible} heuristics.

\begin{definition}[\textbf{Admissibility}]
	A heuristic $h$ is admissible if $h(\bB) \ge 0$ and $h(\bB) \le h^*(\bB)$, where $h^*(\bB)$ is the minimum number of data that must removed from $\cC(\bB)$ to make the remaining data feasible.
\end{definition}

Note that setting $e(\bB) = l(\bB)$ (i.e., $h(\bB) = 0$) for all $\bB$ reduces A* search to breadth-first search (BFS). With an admissible heuristic, A* search is guaranteed to always find $\bB^*$ before other sub-optimal feasible bases (see~\cite{Chin17} for the proof). Algorithm~\ref{alg:hins} describes the heuristic $h_{ins}$ used in~\cite{Chin17}.

\begin{algorithm}
	\begin{algorithmic}[1]
		\Require $\bB$
		\State If $f(\bB)\le\epsilon$, return 0.
		\State $\oO \leftarrow \emptyset$.
		\While{$f(\bB) > \epsilon$}\label{line:hinsF1}
		\State $\oO \leftarrow \oO \cup \bB$, $\bB \leftarrow \tau(\cC(\bB) \backslash \bB)$.
		\EndWhile
		\State $h_{ins} \leftarrow 0$, $\fF \leftarrow \cC(\bB)$.
		\For{each $\bB \in \oO$}
		\For{each $\sss \in \bB$}
		\State $\bB' \leftarrow \tau(\fF\cup\{\sss\})$.
		\If{$f(\bB')\le \epsilon$}\label{line:hinsF2}
		\State $\fF \leftarrow \fF \cup \{\sss\}$.
		\Else
		\State $h_{ins} \leftarrow h_{ins}+1$, $\fF \leftarrow \fF \cup \{\sss\} \backslash \bB'$.
		\EndIf
		\EndFor
		\EndFor\\
		\Return $h_{ins}$.
	\end{algorithmic}
	\caption{Admissible heuristic $h_{ins}$ for A* tree search}\label{alg:hins}
\end{algorithm}

Intuitively, the algorithm for $h_{ins}$ removes a sequence of bases in the first round of iteration until a feasible subset ${\fF \subseteq \cC(\bB)}$ is found. After that, the algorithm iteratively inserts each removed basis point $\sss$ back into $\fF$. If the insertion of $\sss$ makes $\fF$ infeasible, $\tau(\fF\cup\{\sss\})$ is removed from the expanded $\fF$ and the heuristic value $h_{ins}$ is increased by 1.

The admissibility of $h_{ins}$ is proved in~\cite[Theorem 4]{Chin17}. In brief, denote $\fF^*$ as the largest feasible subset of $\cC(\bB)$. If $\fF\cup\{\sss\}$ is infeasible, $\tau(\fF\cup\{\sss\})$ must contain at least one point in $\fF^*$. Since we only add 1 to $h_{ins}$ when this happens, then $h^*(\bB) \ge h_{ins}(\bB)$.

\subsection{Avoiding redundant node expansions}

Algorithm~\ref{alg:A*} employs two strategies to avoid redundant node expansions. In Line~\ref{line:A*-BP}, before expanding $\bB$, a fast heuristic called True Outlier Detection (TOD)~\cite{Chin17} is used to attempt to identify and remove true outliers from $\bB$ (more details in Sec.~\ref{sec:DIBP}), which has the potential to reduce the size of the branch starting from $\bB$. In Line~\ref{line:A*-Rep}, a repeated basis check heuristic is performed to prevent bases that have been explored previously to be considered again (details in Sec.~\ref{sec:rep}).

Our main contributions are two new strategies that improve upon the original methods above, as we will describe in Secs.~\ref{sec:rep} and~\ref{sec:DIBP}. In each of the sections, we will first carefully analyze the weaknesses of the existing strategies. Sec.~\ref{sec:main} will then put our new strategies in an overall algorithm. Sec.~\ref{sec:exp} presents the results.

\section{Non-adjacent path avoidance}\label{sec:rep}

\begin{figure*}[t]\centering 
	\subfigure[Root node $\bB_{root}$.]{\includegraphics[width=0.24\textwidth]{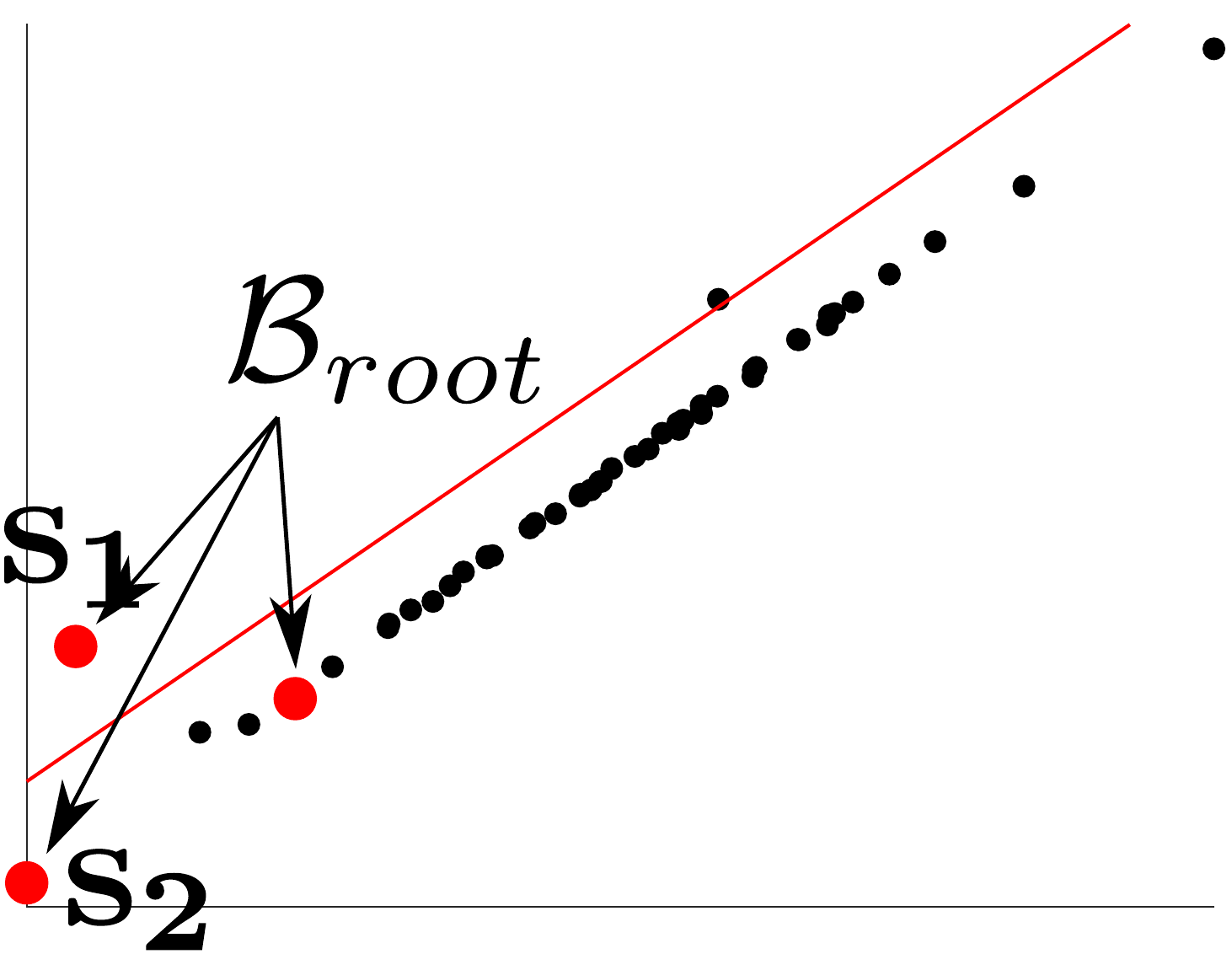}\label{subfig:nonAdjacent1}}
	\subfigure[Level-1 node $\bB$.]{\includegraphics[width=0.24\textwidth]{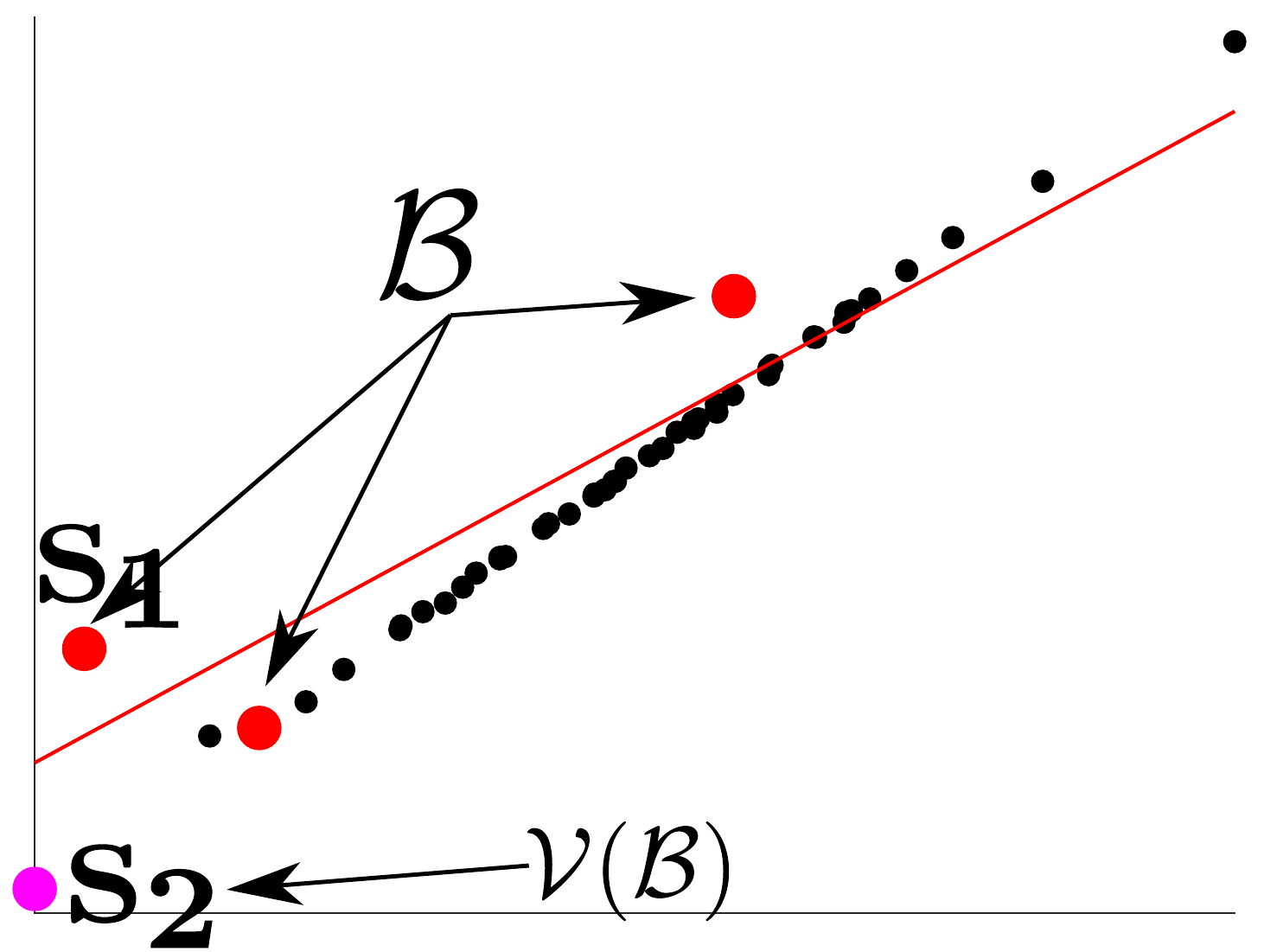}\label{subfig:nonAdjacent2}}
	\subfigure[Level-1 node $\bB$. $\sss_2\in\cC(\bB)$.]{\includegraphics[width=0.24\textwidth]{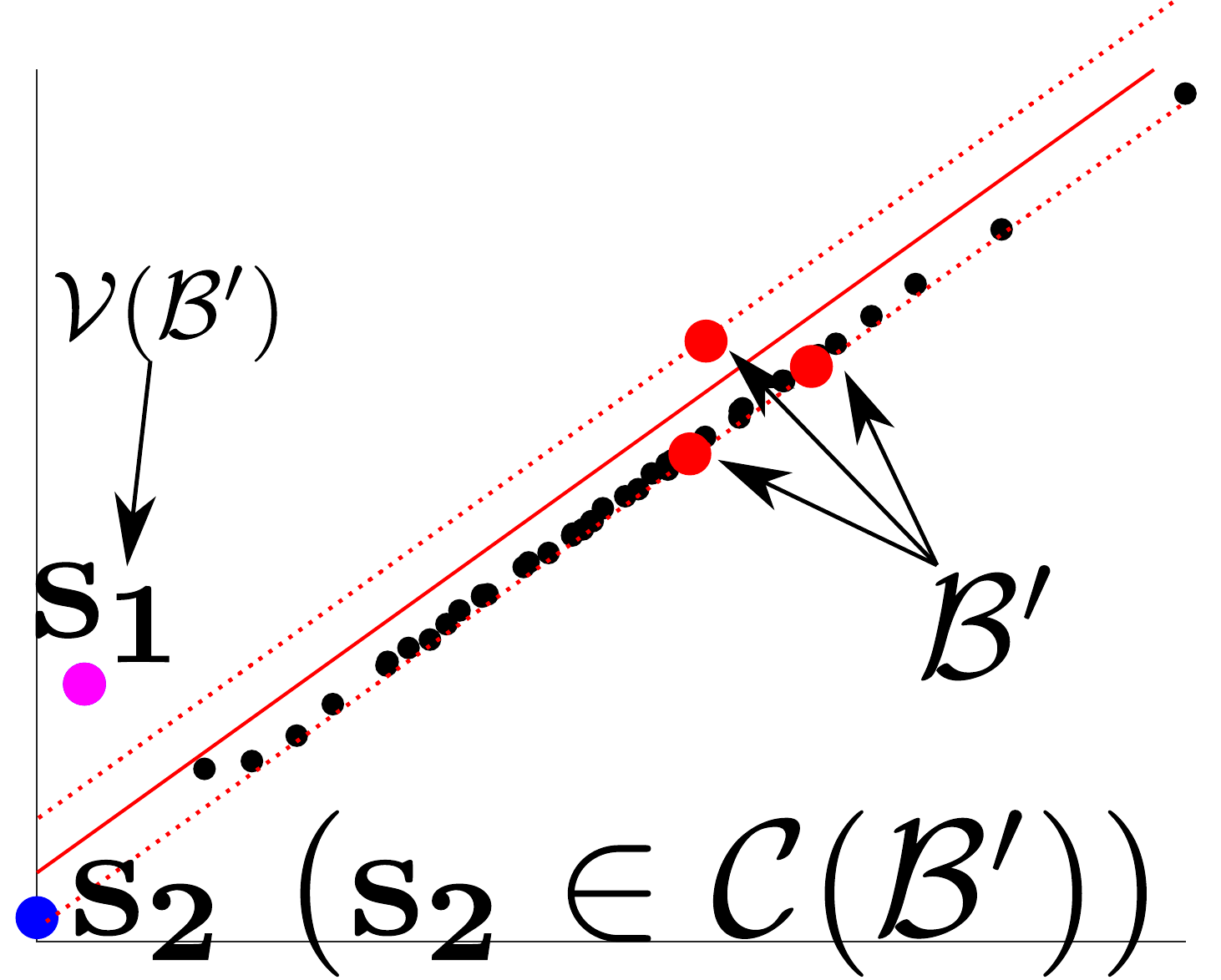}\label{subfig:nonAdjacent3}}
	\subfigure[Tree structure]{\includegraphics[width=0.24\textwidth]{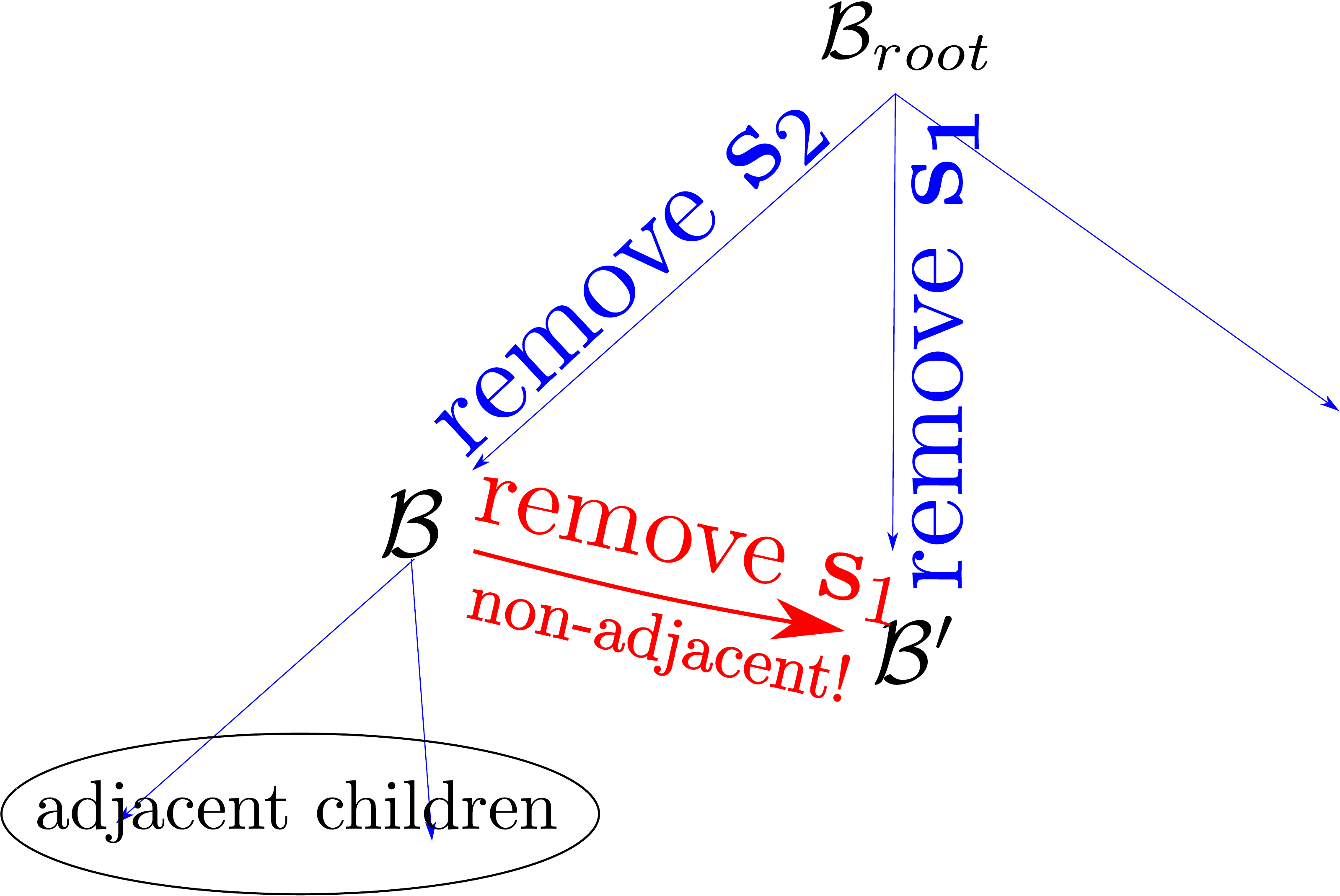}\label{subfig:nonAdjacent4}}
	\caption{(a--c) Path between non-adjacent bases ($\bB \rightarrow \bB'$). $\bB'$ can be generated from both $\bB$ and $\bB_{root}$, but it is \emph{not} adjacent to $\bB$ since $l(\bB') = l(\bB)$. Note that Line~\ref{line:A*-Rep} in Algorithm~\ref{alg:A*} cannot avoid this non-adjacent path since $\vV(\bB')\cup\{\sss_2\} = \{\sss_1, \sss_2 \} \neq \vV(\bB') = \{\sss_1 \}$. Panel (d) shows the relationship between the three bases during tree search. In the proposed Non-Adjacent Path Avoidance (NAPA) strategy, the path drawn in red is not followed. As we will show in Sec.~\ref{sec:exp}, this simple idea provides a massive reduction in runtime of A* tree search.}\label{fig:nonAdjacent}
\end{figure*}


Recall Definition~\ref{def:basisAdj} on adjacency: for $\bB$ and $\bB'$ to be adjacent, their violation sets $\vV(\bB)$ and $\vV(\bB')$ must differ by one point; in other words, it must hold that
\begin{align}
|l(\bB') - l(\bB)| = 1.
\end{align}
Given a $\bB$, Line~\ref{line:A*-expand} in Algorithm~\ref{alg:A*} generates an adjacent ``child" basis of $\bB$ by removing a point $\sss$ from $\bB$ and solving the minimax problem~\eqref{obj:minimax}	on $\cC(\bB)\backslash\{\sss\}$. In this way,
\begin{align}
l(\bB') = l(\bB) + 1.
\end{align}
Iterating the $\sss$ to be removed thus generates all the adjacent child bases of $\bB$, which allows the tree to be explored.

However, an important phenomenon that is ignored in Algorithm~\ref{alg:A*} is, while the above process generates all the adjacent child bases of $\bB$, \emph{not all $\bB'$ generated in the process are adjacent child bases}. Figure~\ref{fig:nonAdjacent} shows a concrete example from line fitting~\eqref{eq:linefitting}: from a root node $\bB_{root}$, two child bases $\bB$ and $\bB'$ are generated by respectively removing points $\sss_2$ and $\sss_1$. However, by further removing $\sss_1$ from $\bB$ and solving~\eqref{obj:minimax} on $\cC(\bB\setminus \{ \sss_1 \})$, we obtain $\bB'$ again! Since $l(\bB') = l(\bB)$, these two bases are not adjacent.

In general, non-adjacent paths occur in Algorithm~\ref{alg:A*} when some elements of $\vV(\bB)$ are in $\cC(\bB')$ after solving the minimax problem on $\cC(\bB\setminus \{ \sss \})$. While inserting a non-adjacent $\bB'$ into the queue does not affect global optimality, it does reduce efficiency. This is because the repeated basis check heuristic in Algorithm~\ref{alg:A*} assumes that the level of the child node $\bB'$ is always lower than the parent $\bB$ by 1; this assumption does not hold if the generated basis $\bB'$ is not adjacent. More formally, if $\bB'$ is not adjacent to $\bB$, then
\begin{align}
\vV(\bB) \cup \{ \sss \} \ne \vV(\bB')
\end{align}
and the repeated basis check in Line~\ref{line:A*-BP} in Algorithm~\ref{alg:A*} fails. Since the same $\bB'$ could be generated from its ``real" parent (e.g., in Figure~\ref{fig:nonAdjacent}, $\bB'$ was also generated by $\bB_{root}$), the same basis can be inserted into the queue more than once.



Since tree search only needs adjacent paths, we can safely skip traversing any non-adjacent path without affecting the final solution. To do this, we propose a Non-Adjacent Path Avoidance (NAPA) strategy for A* tree search; see Fig.~\ref{subfig:nonAdjacent4}. Given a basis $\bB$, any non-adjacent basis generated from it cannot have a level that is higher than $l(\bB)$. Therefore, we can simply discard any newly generated basis $\bB'$ (Line~\ref{line:A*-expand}) if $l(\bB') \le l(\bB)$. Though one redundant minimax problem~\eqref{obj:minimax} still needs to be solved when finding $\bB'$, a much larger cost for computing $e(\bB')$ (which requires to solve multiple problems~\eqref{obj:minimax}) is saved along with all the computation required for traversing the children of $\bB'$. The effectiveness of this strategy will be demonstrated later in Sec.~\ref{sec:exp}.

\section{Dimension-insensitive branch pruning}\label{sec:DIBP}

Our second improvement to A* tree search lies in a new branch pruning technique. We first review the original method (TOD) and then describe our new technique.

\subsection{Review of true outlier detection (TOD)}

\begin{figure}[t]\centering 
	\subfigure[TOD]{\includegraphics[width=0.48\columnwidth]{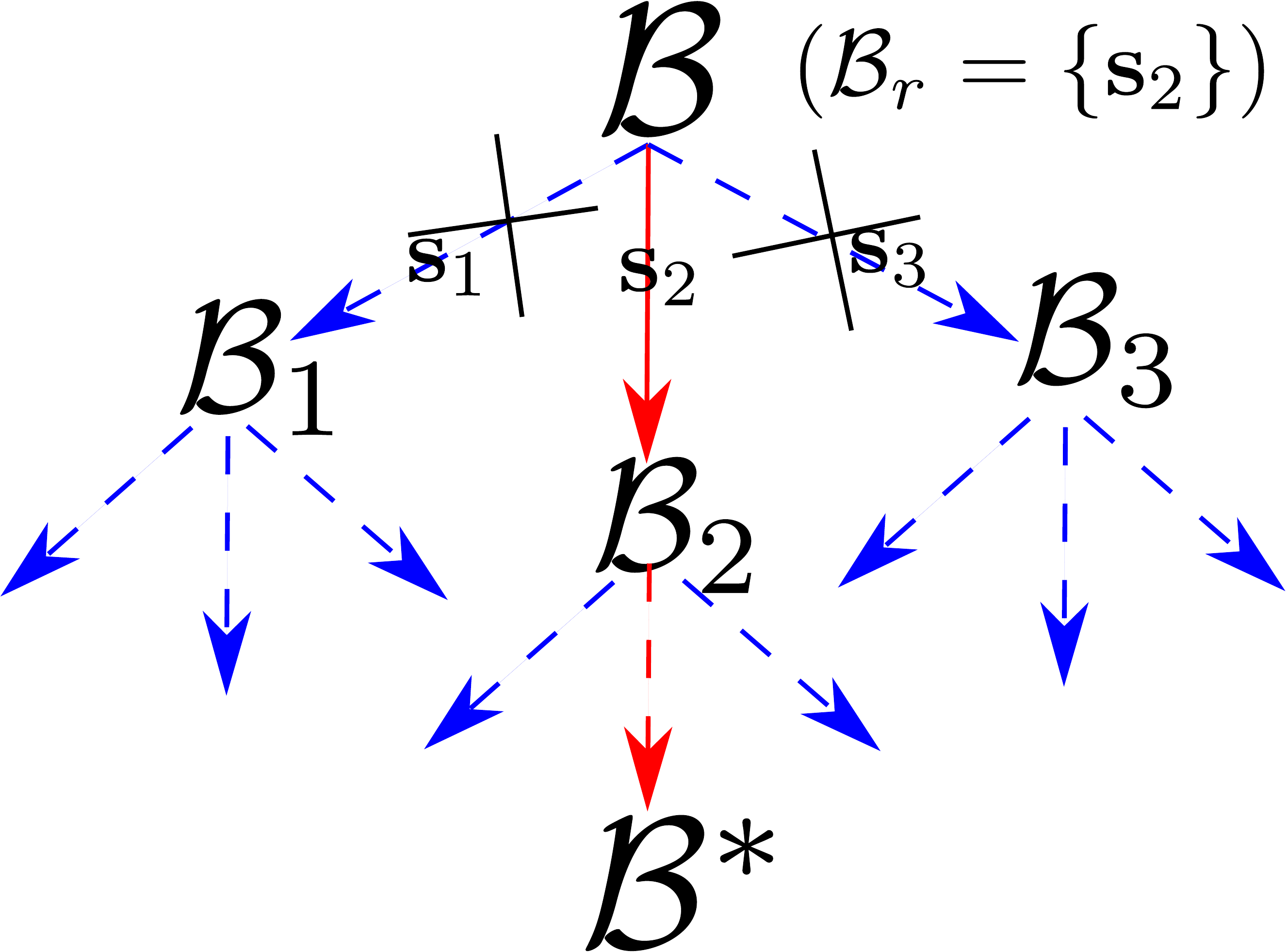}\label{fig:TODIdea}}\hfill
	\subfigure[DIBP]{\includegraphics[width=0.48\columnwidth]{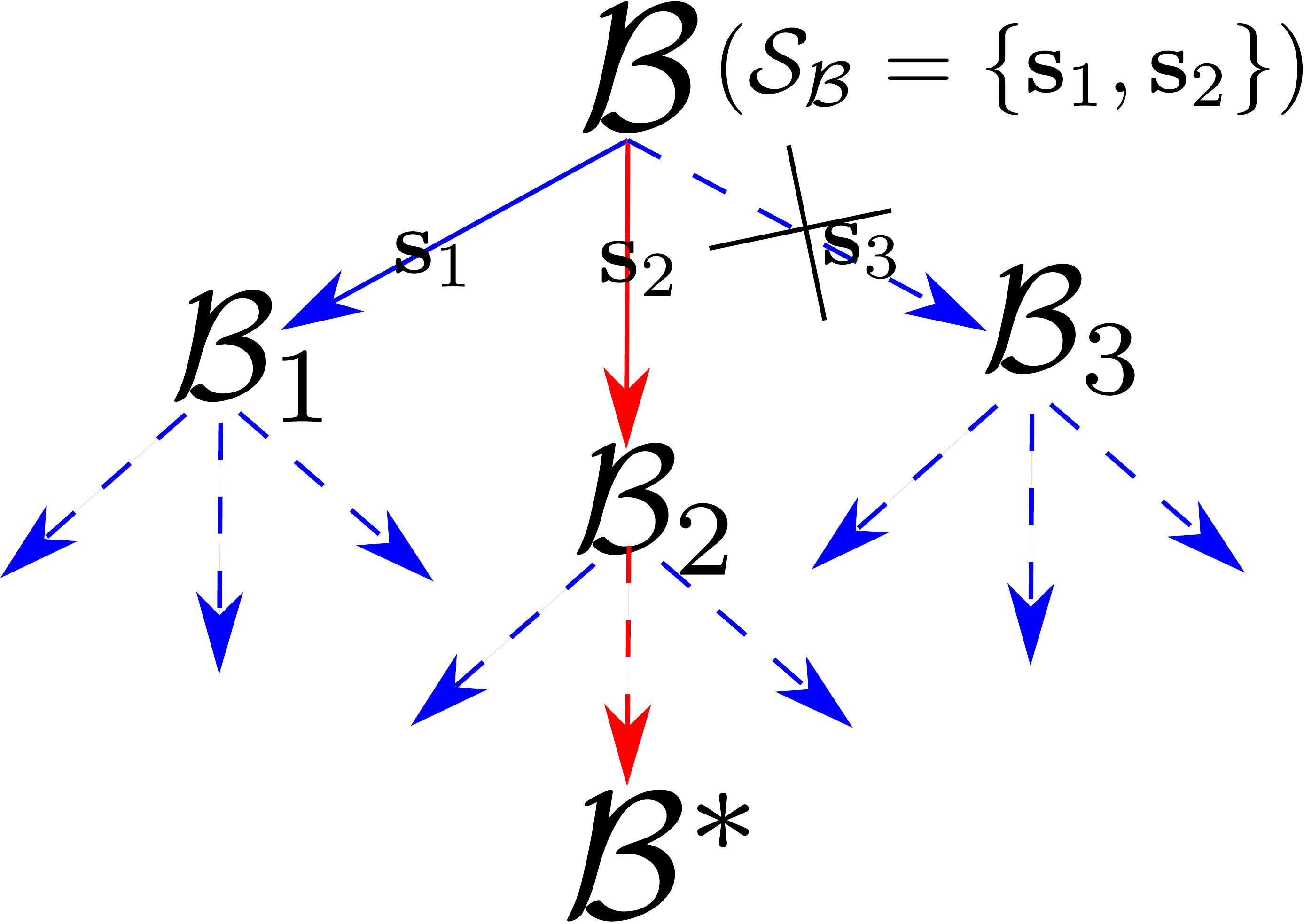}\label{fig:DIBPIdea}}
	\caption{(a) In TOD, on current node $\bB$, if $\sss_2$ is identified as the true outlier, then the shortest path towards a feasible basis $\bB^*$ must pass through $\sss_2$ (path rendered in red). All the the other $|\bB|-1$ branches (leading from $\sss_1$ and $\sss_3$ in this example) can be skipped.		
		(b) In DIBP, instead of attempting to identify a single true outlier, a group $\sS_{\bB}$ that contains at least one true outlier ($\sS_{\bB} = \{\sss_1, \sss_2\}$ in this example) is identified; if this is successful, the other $|\bB|-|\sS_\bB|$ paths (corresponding to $\sss_3$ in this example) can be skipped. DIBP is more effective than TOD because it is easier to reject a subset than a single point as outlier; see Sec.~\ref{sec:DIBP2} for details.}
\end{figure}


Referring to Line~\ref{line:A*-BP} in Algorithm~\ref{alg:A*}~\cite{Chin17}, let $\fF^*$ be the largest feasible subset of $\cC(\bB)$. A point $\sss\in\bB$ is said to be a true outlier if $\sss\notin\fF^*$, otherwise we call it a true inlier. Given an infeasible node $\bB$, one of the elements in $\bB$ must be a true outlier. The goal of TOD is to identify one such true outlier in $\bB$. If $\sss\in\bB$ is successfully identified as a true outlier, we can skip the child generation for all the other points in $\bB$ without hurting optimality, since $\sss$ must be on the shortest path to feasibility via $\bB$; see Fig.~\ref{fig:TODIdea}. If such an $\sss$ can be identified, the reduced subset $\bB_r$ is simply $\{\sss\}$.

The principle of TOD is as follows: define $h^*(\bB | \sss)$ as the \emph{minimum} number of data points that must be removed from $\cC(\bB)$ to achieve feasibility, with $\sss$ forced to be feasible. We can conclude that $\sss\in\bB$ is a true outlier if and only if
\begin{align}\label{eq:TOD*}
&	h^*(\bB | \sss) > h^*(\bB);
\end{align}
see~\cite{Chin17} for the formal proof. Intuitively, if $\sss$ is a true inlier, forcing its feasibility will not change the value of $h^\ast$. On the other hand, if forcing $\sss$ to be feasible leads to the above condition, $\sss$ cannot be a true inlier.

\vspace{-1em}
\paragraph{Bound computation for TOD.}

Unsurprisingly $h^*(\bB | \sss)$ is as difficult to compute as $h^*(\bB)$. To avoid directly computing $h^*(\bB | \sss)$, TOD computes an admissible heuristic $h(\bB|\sss)$ of $h^*(\bB|\sss)$ and an upper bound $g(\bB)$ of $h^*(\bB)$. Given $\sss \in \bB$, $h(\bB|\sss)$ and $g(\bB)$, if
\begin{align}\label{eq:TOD}
&	h(\bB|\sss) > g(\bB),
\end{align}
then it must hold that
\begin{align}
h^*(\bB | \sss) \ge h(\bB|\sss) > g(\bB) \ge h^*(\bB),
\end{align}
which implies that $\sss$ is a true outlier.

As shown in~\cite{Chin17}, $g(\bB)$ can be computed as a by-product of computing $h_{ins}(\bB)$, and $h(\bB|\sss)$ is computed by a constrained version of $h_{ins}$, which we denote as $h_{ins}(\bB|\sss)$. Computing $h_{ins}(\bB|\sss)$ is done by the constrained version of Algorithm~\ref{alg:hins}, where all minimax problems~\eqref{obj:minimax} required to solve are replaced by their constrained versions, which are in the following form:
\begin{subequations}\label{obj:minimaxCon}
	\begin{align}
	& \underset{\btheta}{\text{minimize}}
	& & \max_{\sss_i\in\sS^1} r(\btheta | \sss_i), \\
	& \text{s.t.}
	& & r(\btheta | \sss'_j)\le\epsilon,\ \forall \sss'_j\in\sS'.\label{obj:minmaxCon2}
	\end{align}
\end{subequations}
The only difference between~\eqref{obj:minimaxCon} and~\eqref{obj:minimax} is the constraint that all data in $\sS'$ must be feasible. And similar to~\eqref{obj:minimax},~\eqref{obj:minimaxCon} is also an LP-type problem which can be solved by standard solvers~\cite{eppstein2005quasiconvex}. Similar as in~\eqref{obj:minimax} we also define $f(\sS^1 | \sS')$ as the minimum objective value of~\eqref{obj:minimaxCon} and $\btheta(\sS^1 | \sS')$ as the corresponding optimal solution.

With the above definition, changing Algorithm~\ref{alg:hins} to its constrained version can be simply done by replacing $f(\bB)$ (Line~\ref{line:hinsF1}) and $f(\bB')$ (Line~\ref{line:hinsF2}) by $f(\bB | \{\sss\})$ and $f(\bB' | \{\sss\})$.

\vspace{-1em}
\paragraph{Why is TOD ineffective?}
The effectiveness of TOD in accelerating Algorithm~\ref{alg:A*} depends on how frequent TOD can detect a true outlier. When a true outlier for $\bB$ is detected, TOD  prunes $|\bB|-1$ branches; on the flipside, if TOD cannot identify an $\sss\in\bB$ as the true outlier, the runtime to compute $h_{ins}(\bB|\sss)$ will be wasted. In the worst case where no true outlier is identified for $\bB$, Algorithm~\ref{alg:hins} has to be executed redundantly for $|\bB|$ times. Whether TOD can find the true outlier is largely decided by how well $h_{ins}(\bB|\sss)$ approximates $h^*(\bB|\sss)$.

We now show that $h_{ins}(\bB|\sss)$ is usually a poor estimator of $h^*(\bB|\sss)$. Define $\oO^*(\bB|\sss)$ as the smallest subset that must be removed from $\cC(\bB)\backslash\sss$ to achieve feasibility, with $\sss$ forced to be feasible, i.e., $|\oO^*(\bB|\sss)| = h^*(\bB|\sss)$. Then, $h_{ins}(\bB|\sss)$ and $h^*(\bB|\sss)$ will be different if a basis $\bB_{rem}$ removed during Algorithm~\ref{alg:hins} contains multiple elements in $\oO^*(\bB|\sss)$, since we only add 1 to $h_{ins}$ when actually more than 1 points in $\bB_{rem}$ should be removed. And the following lemma shows that the difference between $h_{ins}(\bB|\sss)$ and $h^*(\bB|\sss)$ will be too large for TOD to be effective if the rate of true outliers in $\cC(\bB)$, i.e., $\frac{h^*(\bB)}{\cC(\bB)}$, is too large.

\begin{lemma}\label{lem:boundH*}
	Condition~\eqref{eq:TOD} is always false when
	\begin{align}\label{eq:boundH*}
	& \frac{h^*(\bB)}{\cC(\bB)} \ge \frac{1}{\phi}\cdot\frac{|\cC(\bB)|-1}{|\cC(\bB)|},
	\end{align}
	where $\phi$ is the \emph{average} size of all $\bB_{rem}$ during Algorithm~\ref{alg:hins}.
\end{lemma}
\begin{proof}
	Since $h_{ins}(\bB|\sss)$ is the number of $\bB_{rem}$ during Algorithm~\ref{alg:hins}, $h_{ins}(\bB|\sss)\cdot\phi \le |\cC(\bB)\backslash\{\sss\}| = |\cC(\bB)|-1$. Hence,
	\begin{align}\label{eq:hins}
	&	h_{ins}(\bB|\sss) \le \frac{|\cC(\bB)|-1}{\phi},
	\end{align}
	Therefore, condition~\eqref{eq:TOD} can never be true if
	\begin{align}\label{eq:lemmaBoundH*}
	& h^*(\bB) \ge \frac{|\cC(\bB)|-1}{\phi}.
	\end{align}
	Dividing both sides of~\eqref{eq:lemmaBoundH*} by $\cC(\bB)$ leads to \eqref{eq:boundH*}.
\end{proof}
Intuitively, when~\eqref{eq:boundH*} happens, there are too many outliers in $\cC(\bB)$ hence too many $\bB_{rem}$ that include multiple elements in $\oO^*(\bB|\sss)$, making $h_{ins}(\bB|\sss)$ too far from $h^*(\bB|\sss)$.

In addition, $\phi$ is positively correlated with $d$, and in the worst case can be $d+1$, which makes TOD sensitive to $d$. Figure~\ref{fig:TOD} shows the effectiveness of TOD as a function of $d$, for problems with linear residual~\eqref{eq:linefitting}. As can be seen, the outlier rate where TOD can be effective reduces quickly with $d$ ($<\!15\%$ when $d\! \ge\! 7$). Note that since $g(\bB)$ is only an estimation of $h^*(\bB)$, the actual range where TOD is effective can be smaller than the region above the dashed line.

\begin{figure}[t]\centering 
	\subfigure{\includegraphics[width=0.99\columnwidth,height=10em]{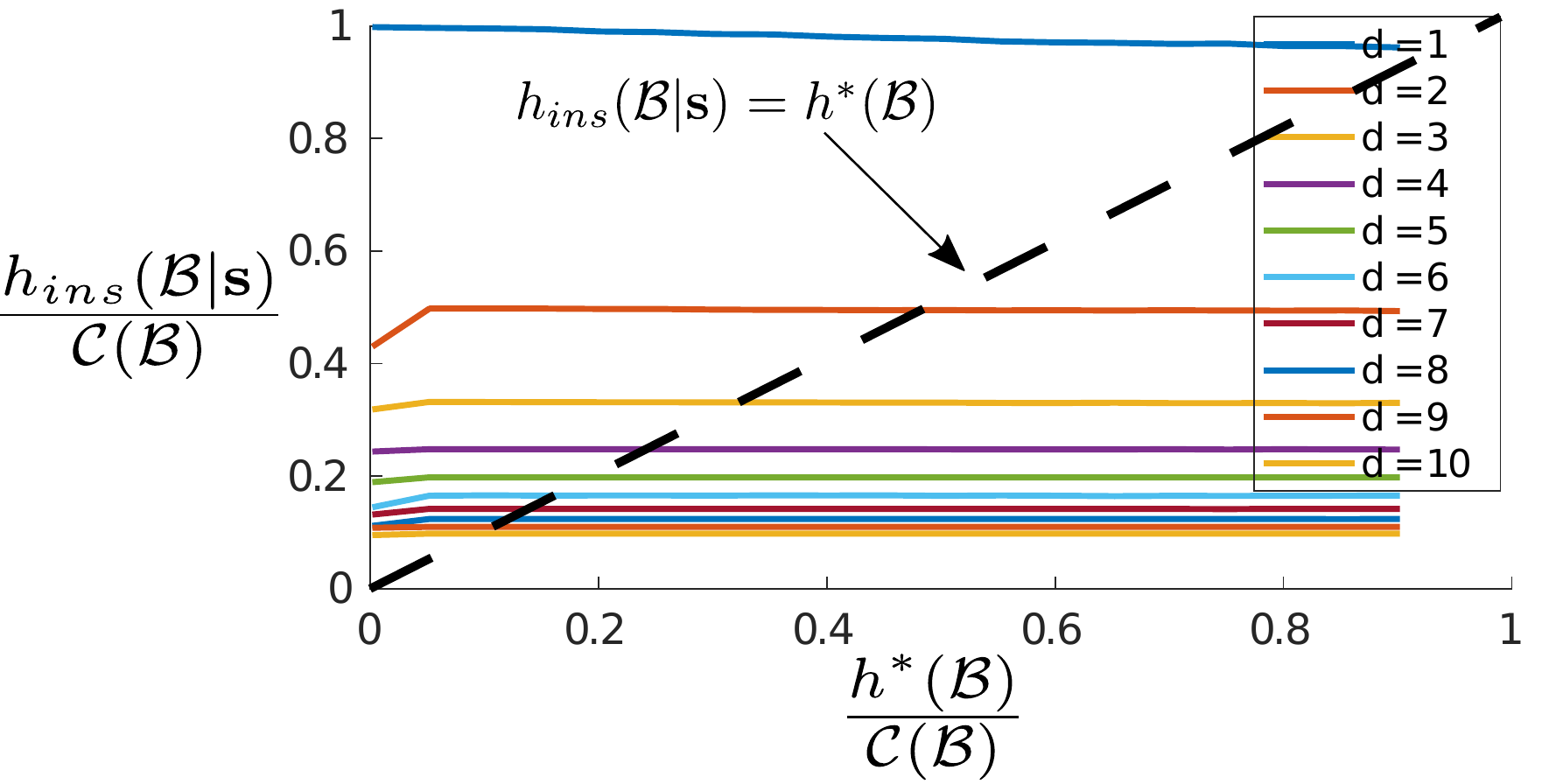}\label{subfig:nonAdjacent1}}
	\caption{Effectiveness of TOD as a function of $d$. All problem instances are generated randomly and each solid curve contains data with true outlier rates $\frac{h^*(\bB)}{\cC(\bB)}$ from 0 to $90\%$. Note that \eqref{eq:boundH*} is true for a $d$ when the solid curve for the $d$ is below the dashed line.}
	\label{fig:TOD}
\end{figure}

\subsection{New pruning technique: DIBP}
\label{sec:DIBP2}


Due to the above limitation, TOD is often not effective in pruning; the cost to carry out Line~\ref{line:A*-BP} in Algorithm~\ref{alg:A*} is thus usually wasted. To address this issue, we propose a more effective branch pruning technique called DIBP (dimension-insensitive branch pruning).

DIBP extends the idea of TOD, where instead of searching for one true outlier, we search for a \emph{subset} $\sS_{\bB}$ of $\bB$ that must contain at least one true outlier. If such a subset can be identified, the children of $\bB$ corresponding to removing points not in $\sS_{\bB}$ can be ignored during node expansion---again, this is because the shortest path to feasibility via $\bB$ must go via $\sS_{\bB}$; Fig.~\ref{fig:DIBPIdea} illustrates this idea.

To find such an $\sS_{\bB}$, we greedily add points from $\bB$ into $\sS_{\bB}$ to see whether enforcing the feasibility of $\sS_{\bB}$ contradicts the following inequality
\begin{align}\label{eq:DIBP}
&	h_{ins}(\bB | \sS_{\bB}) > g(\bB),
\end{align}
which is the extension of~\eqref{eq:TOD}, with $h = h_{ins}$. Similar to $h_{ins}(\bB|\sss)$, $h_{ins}(\bB | \sS_{\bB})$ is computed by the constrained version of Algorithm~\ref{alg:hins} with $\sS' = \sS_{\bB}$ in problem~\eqref{obj:minimaxCon}.

\begin{figure}[t]\centering 
	\subfigure{\includegraphics[width=0.99\columnwidth,height=10em]{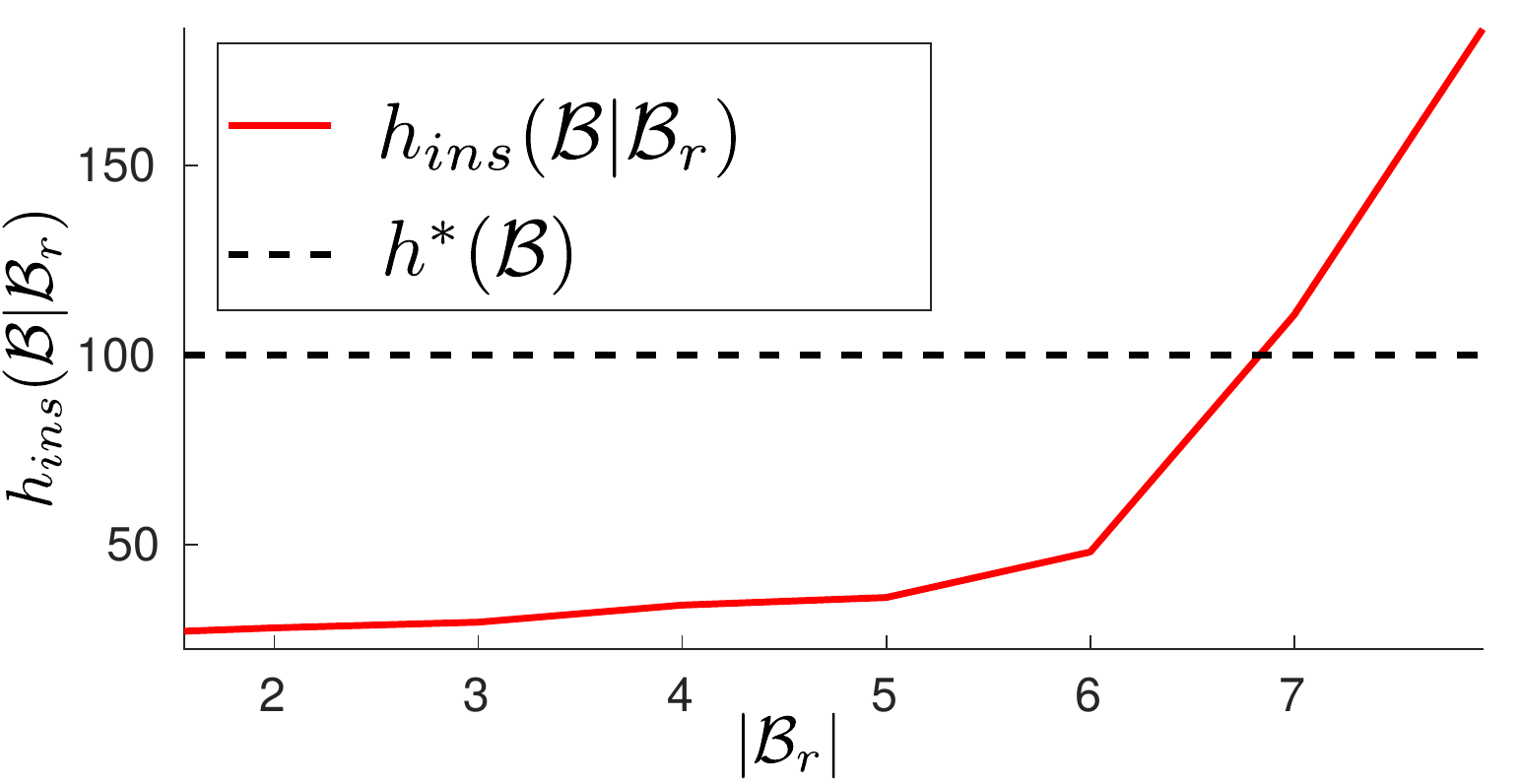}\label{subfig:nonAdjacent1}}
	\caption{Effectiveness of DIBP when $d = 8$. $|\cC(\bB)| = 200$. $h_{ins}(\bB | \sS_{\bB})$ increases stably along with $|\sS_{\bB}|$ and is effective even when the true outlier rate is $90\%$. Though only the $50\%$ case is shown, changing the outlier rate in practice merely affects the values of $h_{ins}(\bB | \sS_{\bB})$ as long as the data distribution is similar.}
	\label{fig:DIBP}
\end{figure}

The insight is that by adding more and more constraints into problem~\eqref{obj:minimaxCon}, the average basis size $\phi$ will gradually reduce, making the right hand side of~\eqref{eq:boundH*} increase until it exceeds the left hand side, so that even with large $d$, branch pruning will be effective with high true outlier rate. Figure~\ref{fig:DIBP} shows the effectiveness of DIBP for an 8-dimensional problem with linear residuals. Observe that $h_{ins}(\bB | \sS_{\bB})$ increases steadily along with $|\sS_{\bB}|$ and can tolerate more than $90\%$ of true outliers when $|\sS_{\bB}| = |\bB| - 1 = 8$.

During DIBP, we want to add true outliers into $\sS_{\bB}$ as soon as possible, since~\eqref{eq:DIBP} can never be true if $\sS_{\bB}$ contains no true outliers. To do so, we utilize the corresponding solution $\btheta_{g(\bB)}$ that leads to $g(\bB)$. During DIBP, the $\sss\in\bB$ with the largest residual $r(\btheta_{g(\bB)}|\sss)$ will be added into $\sS_{\bB}$ first, since a larger residual means a higher chance that $\sss$ is a true outlier. In practice, this strategy often enables DIBP to find close to minimal-size $\sS_{\bB}$.

For problems with linear residuals, we can further compute an adaptive starting value $z(\bB)$ of $|\sS_{\bB}|$, where DIBP can safely skip the first $z(\bB)-1$ computations of $h_{ins}(\bB|\sS_{\bB})$ without affecting the branch pruning result. The value of $z(\bB)$ should be $\max\{1, d+2-\frac{|\cC(\bB)|-1}{g(\bB)}\}$. The reason is demonstrated in the following lemma:
\begin{lemma}\label{lem:adap}
	For problems with linear residuals,~\eqref{eq:DIBP} cannot be true unless
	\begin{align}\label{eq:|Br|}
	|\sS_{\bB}| > d+1 - \frac{|\cC(\bB)|-1}{g(\bB)}.
	\end{align}
\end{lemma}
\begin{proof}
	As in~\eqref{eq:hins}, we have $h_{ins}(\bB|\sS_{\bB}) < \frac{|\cC(\bB)|-1}{\phi}$. To ensure that~\eqref{eq:DIBP} can be true, we must have  $g(\bB)<\frac{|\cC(\bB)|-1}{\phi}$, which we rewrite as
	\begin{align}\label{eq:boundPhi}
	& \phi < \frac{|\cC(\bB)|-1}{g(\bB)}.
	\end{align}
	And for problems with linear residuals,~\eqref{obj:minimaxCon} with $\sS' = \sS_{\bB}$ is a linear program, whose optimal solution resides at a vertex of the feasible polytope~\cite[Chapter 13]{nocedal2006numerical}. This means that for problem~\eqref{obj:minimaxCon}, the
	basis size plus the number of active constraints at the optimal solution must be $d+1$. And since each absolute-valued constraint in~\eqref{obj:minmaxCon2} can at most contribute one active linear constraint, the maximum number of active constraints is $|\sS_{\bB}|$.
	Thus during the computation of $h_{ins}(\bB|\sS_{\bB})$, the average basis size $\phi \ge d+1-|\sS_{\bB}|$. Substituting this inequality into~\eqref{eq:boundPhi} results in~\eqref{eq:|Br|}.
\end{proof}

\section{Main algorithm}
\label{sec:main}

Algorithm~\ref{alg:A*-DIBP} summarizes the A* tree search algorithm with our new acceleration techniques. A reordering is done so that cheaper acceleration techniques are executed first. Specifically, given the current basis $\bB$, we iterate through each element $\sss\in\bB$ and check first whether it leads to a repeated adjacent node and skip $\sss$ if yes (Line~\ref{line:A*New-Rep}). Otherwise, we check whether the node $\bB'$ generated by $\sss$ is non-adjacent to $\bB$ and discard $\bB'$ if yes (Line~\ref{line:A*New-NAPA}). If not, we insert $\bB'$ into the queue since it cannot be pruned by other techniques. After that, we perform DIBP (Line~\ref{line:A*New-DIBP}) and skip the other elements in $\bB$ if condition~\eqref{eq:DIBP} is satisfied. Note that we can still add $\sss$ into $\sS_\bB$ even though it leads to repeated bases. This strategy makes DIBP much more effective in practice.
%
\begin{algorithm}
	\begin{algorithmic}[1]
		\Require $\sS = \{ \sss_i \}^{N}_{i=1}$, threshold $\epsilon$.
		\State Insert $\bB = \tau(\sS)$ with priority $e(\bB)$ into queue $q$.
		\State Initialize hash table $T$ to NULL.
		\While{$q$ is not empty}
		\State Retrieve from $q$ the $\bB$ with the lowest $e(\bB)$.
		\State \textbf{If} $f(\bB) \le \epsilon$ \textbf{then} return $\bB^* = \bB$.
		\State $\sS_{\bB} \leftarrow \emptyset$; Sort $\bB$ descendingly based on $r(\btheta_{g(\bB)}|\sss)$.
		\For{each $\sss \in \bB$}
		\If{indices of $\vV(\bB) \cup \{\sss\}$ do not exist in $T$}\label{line:A*New-Rep}
		\State Hash indices of $\vV(\bB) \cup \{\sss\}$ into T.
		\State $\bB' \leftarrow \tau(\cC(\bB)\backslash\{\sss\})$.
		\If{$l(\bB') > l(\bB)$}.\label{line:A*New-NAPA}
		\State $\sS_{\bB} \leftarrow \sS_{\bB} \cup \{\sss\}$.
		\State Insert $\bB'$ with priority $e(\bB')$ into $q$.
		\State \textbf{If} $|\sS_{\bB}| = |\bB|$ $\vee$ \eqref{eq:DIBP} is true \textbf{then} break.\label{line:A*New-DIBP}
		\EndIf
		\Else
		\State $\sS_{\bB} \leftarrow \sS_{\bB} \cup \{\sss\}$.
		\EndIf
		\EndFor
		\EndWhile
		\State Return error (no inlier set of size greater than $p$).
	\end{algorithmic}
	\caption{A* tree search with NAPA and DIBP}\label{alg:A*-DIBP}
\end{algorithm}

\section{Experiments}
\label{sec:exp}

To demonstrate the effectiveness of our new techniques, we compared the following A* tree search variants:
\begin{itemize}[leftmargin=1em,itemsep=0em,parsep=0em,topsep=0.5em]
	\item Original A* tree search (A*)~\cite{Chin15}.
	\item A* with TOD for branch pruning (A*-TOD)~\cite{Chin17}.
	\item A* with non-adjacenct path avoidance (A*-NAPA).
	\item A*-NAPA with TOD branch pruning (A*-NAPA-TOD).
	\item A*-NAPA with DIBP branch pruning (A*-NAPA-DIBP).
\end{itemize}

All variants were implemented in MATLAB 2018b, based on the original code of A*.
For problems with linear residuals, we use the self-implemented vertex-to-vertex algorithm~\cite{cheney66} to solve the minimax problems~\eqref{obj:minimax} and~\eqref{obj:minimaxCon}. And in the non-linear case, these two problems were solved by the matlab function \texttt{fminimax}. All experiments were executed on a laptop with Intel Core 2.60GHz i7 CPU, 16GB RAM and Ubuntu 14.04 OS.

\subsection{Controlled experiment on synthetic data}

\begin{figure*}
	\begin{minipage}{0.60\textwidth}
		\subfigure[$N = 200$]{\includegraphics[width=0.5\textwidth,height=13.5em]{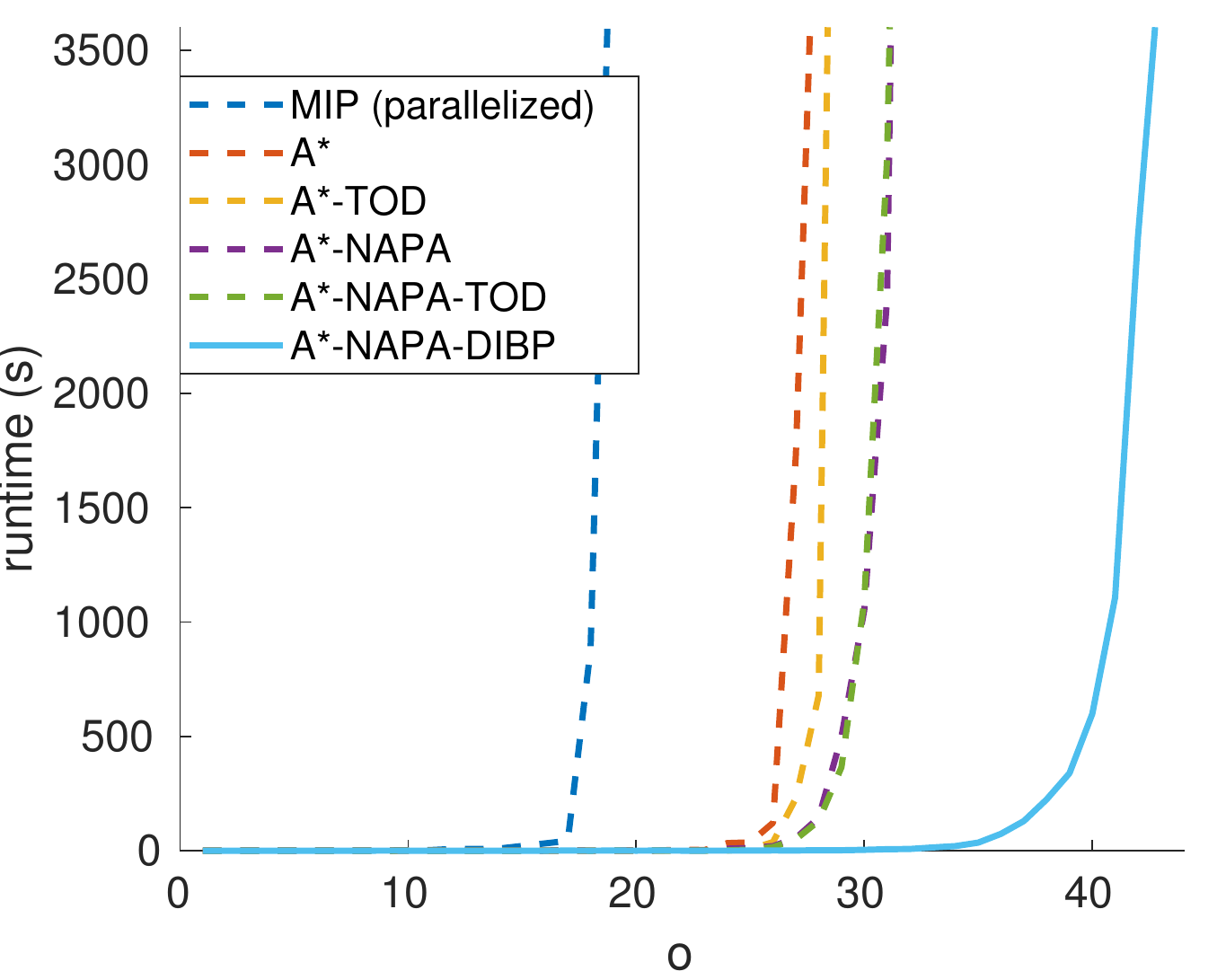}\label{fig:expLinearSynth(b)}}
		\subfigure[$N = 400$]{\includegraphics[width=0.5\textwidth,height=13.5em]{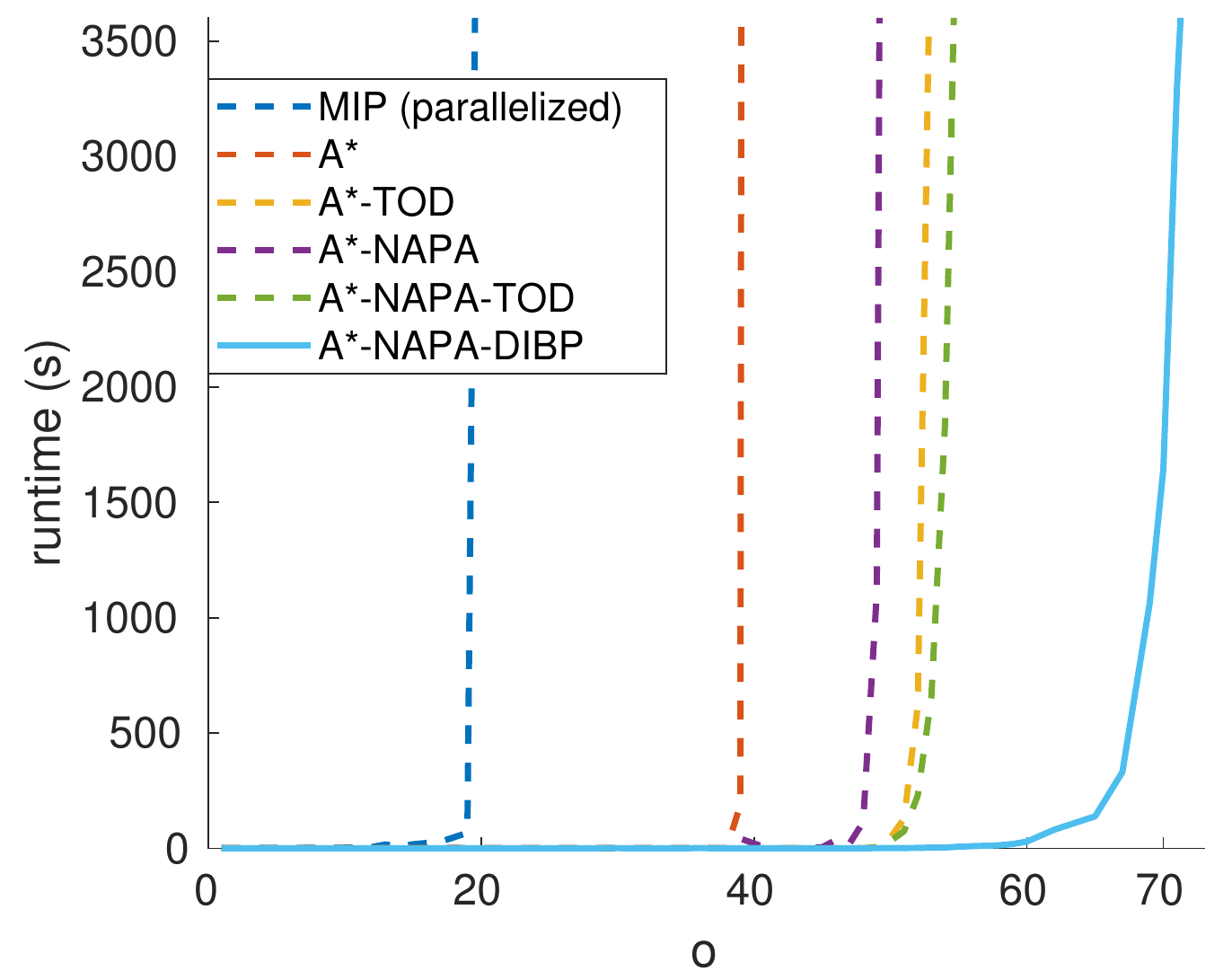}\label{fig:expLinearSynth(C)}}
		\caption{Runtime vs $o$ for robust linear regression on synthetic data. $d = 8$.}
		\label{fig:expLinearSynth}
	\end{minipage}	
	\begin{minipage}{0.39\textwidth}
		\subfigure{\includegraphics[width=1\textwidth]{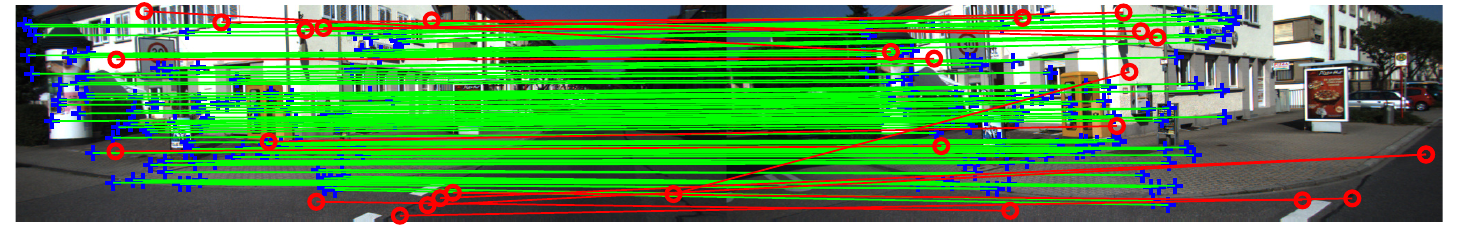}\label{fig:expFund}}
		
		\vspace{-1em}
		
		\subfigure{\includegraphics[width=1\textwidth]{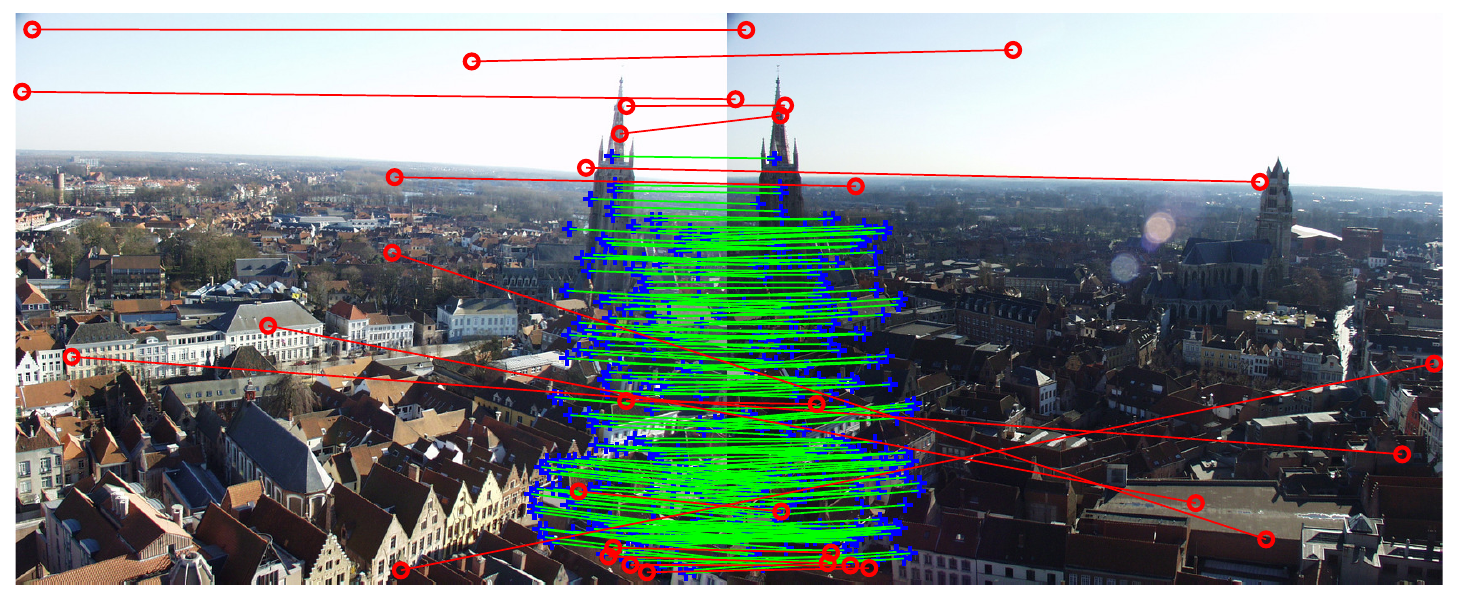}\label{fig:expHomoASTAR}}
		\caption{(Top) Fundamental matrix estimation result of A*-NAPA-DIBP on \texttt{Frame-738-742}. (Bottom) Homography estimation result of A*-NAPA-DIBP on data \texttt{BruggeTower}. The inliers (in green) in the top figure were down-sampled to 100 for clarity.}\label{fig:expReal}
	\end{minipage}	
\end{figure*}

To analyze the effect of $o$ and $N$ to different methods, we conducted a controlled experiment on the 8-dimensional robust linear regression problem with different $N$ and $o$. The residual of linear regression is in the form of~\eqref{eq:linefitting}. To generate data $\sS = \{\aa_i,b_i\}_{i=1}^N$, a random model $\btheta\in\Re^d$ was first generated and $N$ data points that perfectly fit the model were randomly sampled. Then, we randomly picked $N-o$ points as inliers and assigned to the $b_i$ of these points noise uniformly distributed between $[-0.1,0.1]$. Then we assigned to the other $o$ points noise uniformly distributed from $[-5,-0.1)\cup(0.1,5]$ to create a controlled number of outliers. The inlier threshold $\epsilon$ was set to 0.1.

To verify the superior efficiency of tree search compared to other types of globally optimal methods, we also tested the Mixed Integer Programming-based BnB algorithm (MIP)~\cite{zheng11} in this experiment. The state-of-the-art Gurobi solver was used as the optimizer for MIP. MIP was  parallelized by Gurobi using 8 threads, while all tree search methods were executed sequentially.

As shown in Figure~\ref{fig:expLinearSynth}, all A* tree search variants are much faster than MIP, even though MIP was significantly accelerated by parallel computing. Both NAPA and DIBP brought considerable acceleration to A* tree search, which can be verified by the gaps between the variants with and without these techniques. Note that when $N = 200$, A*-NAPA had similar performance with and without TOD, while DIBP provided stable and significant acceleration for all data.

Interestingly, having a larger $N$ made A* tree search efficient for a much larger $o$. This can be explained by condition~\eqref{eq:boundH*}. With the same $o$, a larger $N$ meant a lower true outlier rate, which made~\eqref{eq:boundH*} less likely.

\subsection{Linearized fundamental matrix estimation}\label{sec:expFund}

\begin{table*}[t]\centering
	\ra{1.05}
	\resizebox{1\linewidth}{!}{
		\small{
			\begin{tabular}{|c|c|c|c|c|c|c|c|c|c|c|}
				\hline
				data & \multicolumn{2}{c|}{\texttt{Frame-104-108}} & \multicolumn{2}{c|}{\texttt{Frame-198-201}} & \multicolumn{2}{c|}{\texttt{Frame-417-420}} & \multicolumn{2}{c|}{\texttt{Frame-579-582}} & \multicolumn{2}{c|}{\texttt{Frame-738-742}} \\
				$d = 8$ & \multicolumn{2}{c|}{$o = 13$ $(o_{LRS} = 23)$; $N = 302$} & \multicolumn{2}{c|}{$o = 13$ $(o_{LRS} = 19)$; $N = 309$} & \multicolumn{2}{c|}{$o = 19$ $(o_{LRS} = 23)$; $N = 385$} & \multicolumn{2}{c|}{$o = 22$ $(o_{LRS} = 25)$; $N = 545$} & \multicolumn{2}{c|}{$o = 14$ $(o_{LRS} = 32)$; $N = 476$} \\
				\hline
				& NUN/NOBP & runtime (s) & NUN/NOBP & runtime (s) & NUN/NOBP & runtime (s) & NUN/NOBP & runtime (s) & NUN/NOBP & runtime (s) \\
				\hline
				A* 	& 163232/0 & $>6400$ & 169369/0 & $>6400$ & 144560/0 & $>6400$ & 136627/0 & $>6400$ & 160756/0 & $>6400$ \\
				A*-TOD 	& 134589/119871 & $>6400$ & 129680/126911 & $>6400$ & 80719/92627 & 3712.99 & 55764/58314 & 2709.21 & 49586/50118 & 1729.34 \\
				A*-NAPA 	& 35359/0 & 561.81 & 23775/0 & 351.07 & 175806/0 & 5993.68 & 147200/0 & $>6400$ & 29574/0 & 471.15 \\
				A*-NAPA-TOD 	& 33165/22275 & 770.08 & 19308/13459 & 451.39 & 15310/10946 & 429.06 & 15792/12073 & 576.82 & 14496/10752 & 373.36 \\
				A*-NAPA-DIBP  & 205/311 & \textbf{7.63} & 105/160 & \textbf{3.88} & 172/216 & \textbf{6.85} & 60/84 & \textbf{3.49} & 52/77 & \textbf{2.00} \\
				\hline
				A*-NAPA-DIBP vs & best previous method & faster by & best previous method & faster by & best previous method & faster by & best previous method & faster by & best previous method & faster by \\
				previous best method & A*/A*-TOD & $\mathbf{>839}$\textbf{x} & A*/A*-TOD & $\mathbf{>1648}$\textbf{x} & A*-TOD & $\mathbf{541}$\textbf{x} & A*-TOD & $\mathbf{775}$\textbf{x} & A*-TOD & $\mathbf{864}$\textbf{x} \\
				\hline
			\end{tabular}
		}
	}
	\caption{Linearized fundamental matrix estimation result. The names of the data are the image indices in the sequence. $o_{LRS}$ is the \emph{estimated} outlier number returned by LO-RANSAC. NUN: number of unique nodes generated. NOBP: number of branch pruning steps executed. The last row shows how much faster A*-NAPA-DIBP was, compared to the fastest previously proposed variants (A* and A*-TOD).}\label{tab:resultFund}
\end{table*}

%
Experiments were also conducted on real data. We executed all tree seach variants for linearized fundamental matrix estimation~\cite{Chin17}, which used the algebaric error~\cite[Sec.11.3]{hartley2003multiple} as the residual and ignored the non-convex rank-2 constraints. 5 image pairs (the first 5 crossroads) were selected from the sequence \texttt{00} of the KITTI Odometry dataset~\cite{Geiger2012CVPR}. For each image pair, the input was a set of SIFT~\cite{lowe1999object} feature matches generated using VLFeat~\cite{vedaldi2010vlfeat}. The inlier threshold $\epsilon$ was set to $0.03$ for all image pairs.

The result is shown in Table~\ref{tab:resultFund}. We also showed the number of unique nodes (NUN) generated and the number of branch pruning steps (NOBP) executed before the termination of each algorithm. A*-NAPA-DIBP found the optimal solution in less than 10s for all data, while A* and A*-TOD often failed to finish in 2 hours. A*-NAPA-DIBP was faster by more than 500 times on all data compared to the fastest method among A* and A*-TOD. For the effectiveness of each technique, applying NAPA to A* often resulted in more than 10x acceleration. And applying DIBP further sped up A*-NAPA by more than 1000x on challenging data (e.g. $\texttt{Frame-198-201}$). This significant acceleration is because many elements in $\sS_{\bB}$ were the ones that led to redundant nodes, which made most non-redundant paths effectively pruned. TOD was much less effective than DIBP and introduced extra runtime to A*-NAPA on \texttt{Frame-104-108} and \texttt{Frame-198-201}. We also attached $o_{LRS}$, the \emph{estimated} number of outliers returned from LO-RANSAC~\cite{chum2003locally}, which is an effective RANSAC variant. None of the LO-RANSAC results were optimal. A visualization of the tree search result is shown in Figure~\ref{fig:expReal}.

\subsection{Homography estimation (non-linear)}

\begin{table*}[!htb]\centering
	\ra{1.05}
	\resizebox{1\linewidth}{!}{
		\small{
			\begin{tabular}{|c|c|c|c|c|c|c|c|c|c|c|}
				\hline
				data & \multicolumn{2}{c|}{\texttt{Adam}} & \multicolumn{2}{c|}{\texttt{City}} & \multicolumn{2}{c|}{\texttt{Boston}} & \multicolumn{2}{c|}{\texttt{Brussels}} & \multicolumn{2}{c|}{\texttt{BruggeTower}} \\
				$d = 8$ & \multicolumn{2}{c|}{$o = 38$ $(o_{LRS} = 40)$; $N = 282$} & \multicolumn{2}{c|}{$o = 19$ $(o_{LRS} = 22)$; $N = 87$} & \multicolumn{2}{c|}{$o = 43$ $(o_{LRS} = 44)$; $N = 678$} & \multicolumn{2}{c|}{$o = 9$ $(o_{LRS} = 25)$; $N = 231$} & \multicolumn{2}{c|}{$o = 17$ $(o_{LRS} = 26)$; $N = 208$} \\
				\hline
				& NUN/NOBP & runtime (s) & NUN/NOBP & runtime (s) & NUN/NOBP & runtime (s) & NUN/NOBP & runtime (s) & NUN/NOBP & runtime (s) \\
				\hline
				A* 	& 224/0 & 538.91 & 7072/0 & $>6400$ & 406/0 & 2455.03 & 397/0 & 437.25 & 5003/0 & $>6400$ \\
				A*-TOD 	& 38/37 & \textbf{156.98} & 462/514 & 910.51 & 7/6 & \textbf{74.63} & 359/281 & 499.77 & 333/260 & 298.39 \\
				A*-NAPA 	& 168/0 & 404.77 & 6481/0 & $>6400$ & 234/0 & 1284.14 & 264/0 & 268.85 & 3731/0 & 4740.68 \\
				A*-NAPA-TOD 	& 38/37 & \textbf{156.98} & 286/241 & 485.36 & 7/6 & \textbf{74.63} & 249/191 & 297.91 & 201/151 & 161.95 \\
				A*-NAPA-DIBP  & 38/37 & \textbf{156.98} & 34/40 & \textbf{64.44} & 7/6 & \textbf{74.63} & 30/42 & \textbf{50.13} & 40/48 & \textbf{68.20} \\
				\hline
				A*-NAPA-DIBP vs & best previous method & faster by & best previous method & faster by & best previous method & faster by & best previous method & faster by & best previous method & faster by \\
				previous best method & A*-TOD & same runtime & A*-TOD & $\mathbf{13.1}$\textbf{x} & A*-TOD & same runtime & A* & $\mathbf{7.7}$\textbf{x} & A*-TOD & $\mathbf{3.4}$\textbf{x} \\
				\hline
			\end{tabular}
		}
	}
	\caption{Homography estimation result. $o_{LRS}$ is the \emph{estimated} outlier number returned by LO-RANSAC. NUN: number of unique nodes generated. NOBP: number of branch pruning steps executed. The last row shows how much faster A*-NAPA-DIBP was, compared to the fastest previously proposed variants (A* and A*-TOD).}\label{tab:resultHomo}
\end{table*}


To test all methods on non-linear problems, another experiment for homography estimation~\cite{hartley2003multiple} was done on ``homogr'' dataset\footnote{\url{http://cmp.felk.cvut.cz/data/geometry2view/index.xhtml}}. As before, we picked 5 image pairs, computed the SIFT matches and used them as the input data. The transfer error in one image~\cite{hartley2003multiple} was used as the residual, which was in the form of~\eqref{eq:pseudoConvex}. $\epsilon$ was set to 4 pixels.

Table~\ref{tab:resultHomo} shows the result of all methods. Compared to the linear case, solving non-linear minimax problems~\eqref{obj:minimax} and~\eqref{obj:minimaxCon} was much more time-consuming (can be $100$x slower with \texttt{fminimax}). Thus with similar NUN and NOBP, the runtime was much larger. However, the value of $\phi$ in the non-linear case was usually also much smaller, which made the heuristic $h_{ins}$ and in turn all branch pruning techniques much more effective than in the linear case. And for easy data such as \texttt{Boston} and \texttt{Adam}, performing either TOD or DIBP was enough to achieve the highest speed. Nonetheless, DIBP was still much more effective than TOD on other data. And DIBP never slowed down the A* tree search as TOD sometimes did (e.g., in \texttt{Brussels}). A*-NAPA-DIBP remained fastest on all image pairs. An example of the visual result is provided in Figure~\ref{fig:expReal}.

\vspace{-0.5em}

\section{Conclusion}

\vspace{-0.5em}

We presented two new acceleration techniques for consensus maximization tree search. The first avoids redundant non-adjacent paths that exist in the consensus maximization tree structure.	The second makes branch pruning much less sensitive to the problem dimension, and therefore much more reliable. The significant acceleration brought by the two techniques contributes a solid step towards practical and globally optimal consensus maximization.

\vspace{-1em}
\paragraph{Acknowledgements.} We thank Dr. Nan Li for his valuable suggestions.

\clearpage

\balance

{\small
	\bibliographystyle{ieee_fullname}
	\bibliography{treeSearch}
}

\end{document}